\PassOptionsToPackage{bookmarks={false}}{hyperref}
\documentclass[nonacm,sigconf]{acmart}
\renewcommand\footnotetextcopyrightpermission[1]{} 
\AtBeginDocument{%
  \providecommand\BibTeX{{%
    \normalfont B\kern-0.5em{\scshape i\kern-0.25em b}\kern-0.8em\TeX}}}

\usepackage{amsmath,amsfonts,amsthm}
\usepackage[ruled]{algorithm}
\usepackage[noend]{algpseudocode}
\usepackage[normalem]{ulem}
\usepackage{graphicx}
\usepackage{textcomp}
\usepackage{xcolor}
\usepackage{color}
\usepackage{etoolbox}
\usepackage{subcaption}
\usepackage{multirow}

\usepackage{amsmath,amsfonts,bm}

\def\eqref#1{equation~\ref{#1}}

\def\1{\bm{1}}

\def\vz{{\bm{z}}}

\def\mH{{\bm{H}}}
\def\mI{{\bm{I}}}

\def\mU{{\bm{U}}}

\def\mW{{\bm{W}}}

\def\sD{{\mathbb{D}}}

\def\sS{{\mathbb{S}}}

\newcommand{\E}{\mathbb{E}}

\newcommand{\R}{\mathbb{R}}

\usepackage{hyperref}
\usepackage{url}
\usepackage[english]{babel}
\newtheorem{theorem}{Theorem}

\newbool{inccomment}
\setbool{inccomment}{true}

\newcommand\blfootnote[1]{%
  \begingroup
  \renewcommand\thefootnote{}\footnote{#1}%
  \addtocounter{footnote}{-1}%
  \endgroup
}

\begin{document}
\title{HERO: Hessian-Enhanced Robust Optimization for Unifying and Improving Generalization and Quantization Performance}

\author{Huanrui Yang*, Xiaoxuan Yang*, Neil Zhenqiang Gong and Yiran Chen}
\affiliation{%
  \institution{Duke University}
  \city{Durham}
  \state{NC}
  \country{USA}
}
\email{{huanrui.yang, xy92, neil.gong, yiran.chen}@duke.edu}

\begin{abstract}
With the recent demand of deploying neural network models on mobile and edge devices, it is desired to improve the model's generalizability on unseen testing data, as well as enhance the model's robustness under fixed-point quantization for efficient deployment. Minimizing the training loss, however, provides few guarantees on the generalization and quantization performance. In this work, we fulfill the need of improving generalization and quantization performance simultaneously by theoretically unifying them under the framework of improving the model's robustness against bounded weight perturbation and minimizing the eigenvalues of the Hessian matrix with respect to model weights. We therefore propose HERO, a Hessian-enhanced robust optimization method, to minimize the Hessian eigenvalues through a gradient-based training process, simultaneously improving the generalization and quantization performance. 
HERO enables up to a 3.8\% gain on test accuracy, up to 30\% higher accuracy under 80\% training label perturbation, and the best post-training quantization accuracy across a wide range of precision, including a $>10\%$ accuracy improvement over SGD-trained models for common model architectures on various datasets. 
\end{abstract}
\maketitle
\blfootnote{* Equal contribution.}
\section{Introduction}

The rapid development of deep learning algorithms has seen the emergence of high-performance deep neural network (DNN) models. Models like VGG~\cite{simonyan2014very}, ResNet~\cite{he2016deep}, MobileNet~\cite{sandler2018mobilenetv2}, etc., have been deployed on mobile and edge applications to process data gathered in the wild. Extensive model deployment requires the model to generalize well to unseen data, and to maintain high performance under fixed-precision quantization for memory and computational efficiency on mobile and edge devices~\cite{horowitz20141}.  

In most cases, DNN models are trained following the \textit{empirical risk minimization} (ERM) setting, whose objective is to minimize the model loss $L_{\sS}(\mW)$ induced by weight $\mW$ on the training set $\sS$. However, only minimizing the ERM objective may not lead to an ideal model for practical applications: the model may be overfitted to the training set and have low testing accuracy~\cite{zhang2016understanding,foret2020sharpness}, or be severely degraded by the post-training quantization process when deploying to the real world~\cite{zhou2016dorefa,alizadeh2020gradient}.

Previous work has been contributing empirical methods to improve DNN generalizability or quantization performance individually. 
Methods like weight decay~\cite{krogh1991simple}, batch normalization~\cite{ioffe2015batch}, stochastic model architecture~\cite{srivastava2014dropout,huang2016deep}, and intensive data augmentation~\cite{cubuk2018autoaugment,zhang2017mixup} improve model generalizability, yet they are not contributing to quantization performance~\cite{alizadeh2020gradient}. 
Quantization-aware training~\cite{zhou2016dorefa,polino2018model,yang2021bsq} regains the quantization performance via retraining on a specific quantization precision, yet they fail to perform well when the precision is changed on the fly~\cite{alizadeh2020gradient}, also hurting the generalization performance of the full-precision model. A DNN training method achieving both high generalization accuracy and high quantization robustness is still lacking.

Interestingly, we notice that previous theoretical analysis has shed light on unifying the pursuit of generalization and quantization performance.~\citet{foret2020sharpness} show DNN's generalization gap is related to the model's robustness against $\ell_2$ norm bounded weight perturbation, whereas the robustness against quantization is shown to be equivalent to the robustness against $\ell_\infty$ norm bounded weight perturbation~\cite{alizadeh2020gradient}. However, the first-order approximation used to improve weight perturbation robustness in both~\cite{foret2020sharpness} and~\cite{alizadeh2020gradient} leaves a weak robustness guarantee and makes them only work against one of the $\ell_2$ or $\ell_\infty$ perturbation bound, failing to work on both generalization and quantization performance as we show later.

In this work, we aim to improve DNN generalization and quantization performance simultaneously with a novel optimization method. As discussed in Section~\ref{ssec:unify}, we form our objective as improving the model's robustness against a general $\ell_p$ norm bounded weight perturbation. Further analysis with second-order Taylor expansion in Section~\ref{ssec:bound} unveils that the minimal perturbation strength under both $\ell_2$ and $\ell_\infty$ bound leading to a loss increase can be bounded by the maximum eigenvalue of the Hessian matrix with respect to the weight. Therefore, in Section~\ref{ssec:reg}, we propose an effective way to regularize Hessian eigenvalue. We further derive Hessian-Enhanced Robust Optimization, \textit{HERO}, which efficiently performs the Hessian eigenvalue regularization through a gradient-based optimization process. HERO leads to a better generalization performance and a better robustness to quantization on all precision, as in Section~\ref{ssec:opt}. To the best of our knowledge, HERO is the first to make the following theoretical contributions:
\begin{itemize}
    \item Unifying generalization and quantization performance under the framework of improving the model's robustness against a general $\ell_p$ norm bounded weight perturbation;
    \item Showing the weight perturbation robustness can be improved via regularizing Hessian eigenvalues with respect to the model weights during DNN training;
    \item Deriving gradient update rule to optimize the Hessian eigenvalue regularization, leading to simultaneous improvement on generalization and quantization performance. 
\end{itemize}

Our theoretical analysis is well-supported by the empirical results. 
For generalization, HERO consistently achieves higher test accuracy, 
including a significant $2.58\%$ and $3.78\%$ accuracy gain with MobileNetV2 over SGD on CIFAR-10 and CIFAR-100 datasets, respectively. 
We further show the generalizability achieved by HERO is robust under the presence of label noise, where HERO outperforms SGD by $5\sim 30\%$ on ResNet20 and $2\sim 10\%$ on MobileNetV2 when training on CIFAR-10 with $20\sim 80\%$ label perturbation.
For quantization, HERO provides the best post-training accuracy under a wide range of precision, including a $>10\%$ accuracy improvement over SGD-trained MobileNet and VGG model at ultra-low precision of 4-5 bits. HERO also beats state-of-the-art Gradient $\ell_1$~\cite{alizadeh2020gradient} by a large margin under all quantization schemes.
Additional ablation studies are also provided to verify our theoretical insights.
\section{Related work}

\subsection{Improving Model Generalization}
As recent research utilizes heavily over-parameterized DNN models, it's essential to prevent the model from overfitting to the training set so that it can generalize well to unseen data. Overfitting can be largely resolved via regularization and data augmentation. For regularization, previous work has developed weight decay~\cite{krogh1991simple}, dropout~\cite{srivastava2014dropout}, stochastic depth~\cite{huang2016deep}, etc. As for data augmentation, recent methods explore mixup~\cite{zhang2017mixup}, auto-augmentation~\cite{cubuk2018autoaugment}, etc. However, theoretical understanding of why these methods help model generalization is still lacking. 
Sharpness aware minimization (SAM)~\cite{foret2020sharpness} theoretically links the generalization ability of the model with the model performance under $\ell_2$ norm bounded weight perturbation, and therefore to the smoothness of the loss surface. SAM provides an efficient optimization algorithm to improve generalization, yet the first-order approximations involved weakens its ability to guarantee performance. HERO builds upon the observation of SAM and proposes an efficient way to regularize the loss smoothness via Hessian eigenvalues, leading to a stronger theoretical guarantee on the generalization performance. 
Moreover, the effectiveness of previous methods on quantized models is not well understood, while HERO unifies the pursuit of generalization and quantization, solving both problems simultaneously.

\subsection{Improving Quantization Robustness}
Quantization is essential for deploying a DNN model onto mobile and edge devices, as it saves on-device memory and achieves both run-time speedup and less energy cost~\cite{horowitz20141}. 
Moreover, the dynamic change of power and memory availability on the device would require changing the precision of a pretrained model on the fly~\cite{alizadeh2020gradient}. However, directly quantizing a DNN model to a low precision (less than 8-bit) will lead to a severe accuracy drop. 
Straight-through estimator~\cite{bengio2013estimatingSTE} enables the finetuning of quantized models to regain the lost accuracy~\cite{zhou2016dorefa,polino2018model,yang2021bsq}. However, the resulting model only works on the exact quantization precision it is trained on; modifying the precision requires a lengthy retraining process. 
Others aim to design quantization schemes or rounding functions that can minimize the post-training quantization loss~\cite{banner2018post,zhao2019improving}, yet these methods require extensive analysis of the model architecture and parameter distribution, making it hard to apply on the fly.
The only previous work successfully achieving general robustness against all quantization precision is Gradient $\ell_1$~\cite{alizadeh2020gradient}, which applies $\ell_1$ regularization on the gradient of the model. As this method is based on a first-order approximation to the quantization loss, our later analysis shows it is insufficient to guarantee robustness. HERO further introduces Hessian regularization, which can lead to a stronger guarantee on much higher quantization robustness.
Besides linear quantization, it is possible to design nonlinear or nonuniform quantizer~\cite{zhang2018lq}, yet utilizing those quantizer requires specialized design on the arithmetic unit, which is not supported by common devices. Thus, we focus on the linear uniform quantization in this work.

\subsection{Curvature Regularization}
As we link the problem of generalization and quantization with the model performance under weight perturbation, we take inspiration from the related field of adversarial robustness, where extensive studies have been done towards DNN's robustness against adversarial perturbation on the input~\cite{goodfellow2014explaining,madry2018towards}. One noticeable work is the curvature regularization (CURE), which shows that the robustness against input perturbation can be improved by regularizing the Hessian eigenvalues of the loss function with respect to the input~\cite{moosavi2019robustness}. HERO also applies the Hessian eigenvalue regularization, but is different from CURE as we are working with respect to the model weight, rather than the input. The regularization of HERO needs to be computed on the weight tensors from multiple layers, each having distinct value and gradient ranges. We tackle the challenge of adapting perturbation strength across different layers based on their weight distribution, as introduced in Section~\ref{ssec:reg}. Furthermore, we apply additional first-order regularization to the optimization process as introduced in Section~\ref{ssec:opt}, effectively leading to better generalization and quantization performance.

\section{Theoretical Analysis}

\subsection{Unifying Generalization and Quantization}
\label{ssec:unify}

Here we start with investigating the properties needed for a deep neural network model to have both good generalizability and high quantization performance. 

\noindent\textbf{Bounding Generalization Gap.}
Recently, a theoretical analysis was made by~\citet{foret2020sharpness} on bounding the generalization gap of a deep neural network, which can be stated as:
\begin{theorem}
\label{the:gen}
For any $\rho>0$, with high probability over training set $\sS$ generated from distribution $\sD$,
\begin{equation}
    L_{\sD}(\mW) - L_{\sS}(\mW) \leq \left[ \max_{||\delta||_2 \leq\rho} L_{\sS}(\mW + \delta) - L_{\sS}(\mW) \right] + h(||\mW||_2^2 / \rho^2)
\end{equation}
where $L$ is the loss function, $\mW$ denotes the weight of the model and $h: \R_+ \rightarrow \R_+$ is a strictly increasing function~\citep{foret2020sharpness}.
\end{theorem}
Note that the second term relating to $||\mW||_2^2$ can be effectively minimized during training with weight decay~\cite{krogh1991simple}, so the generalization gap is largely bounded by the model's performance under a weight perturbation $\delta$ bounded by its $\ell_2$ norm.

\noindent\textbf{Bounding Quantization Loss.}
In the meantime, the post-training quantization process can also be considered as a process of perturbing the model weights. Here we focus on the typical setting of a linear uniform weight quantization~\cite{polino2018model}, where the weight distribution is separated into $2^n$ uniform-sized bins, and each bin is rounded into a $n$-bit quantized value. Suppose the quantization bin has a width of $\Delta$, the rounding function will change each element of the weight by at most $\Delta/2$. So the weight perturbation induced by quantization is bounded by the $\ell_\infty$ norm, as $||\delta||_\infty := ||\mW_q - \mW||_\infty \leq \Delta/2$, where $\mW$ and $\mW_q$ denote the original and quantized weight, respectively. Therefore we can bound the loss increase introduced by quantization as:
\begin{theorem}
\label{the:quant}
For a linear uniform quantization with a bin width $\Delta = 2\rho$, we have
\begin{equation}
    L_{\sS}(\mW_q) - L_{\sS}(\mW) \leq \left[ \max_{||\delta||_\infty \leq\rho} L_{\sS}(\mW + \delta) - L_{\sS}(\mW) \right],
\end{equation}
\end{theorem}
which is bounded by the model's performance under a weight perturbation $\delta$ bounded by its $\ell_\infty$ norm.

\noindent\textbf{Unifying the Bounds.}
With the analysis on Theorem~\ref{the:gen} and~\ref{the:quant}, we can unify the pursuit of generalization and quantization performance as understanding how the model loss changes under a general $\ell_p$ norm bounded weight perturbation. Specifically, we can derive lower bounds for the minimal strength needed for perturbation $\delta$ to induce an increase $c$ in the model loss as: 
\begin{equation}
\label{eq:perturb}
    \delta^* := \arg\min_{\delta} ||\delta||_p \ \ \  s.t.\ L_{\sS}(\mW+\delta)-L_{\sS}(\mW) \geq c.
\end{equation}
A larger lower bound on $||\delta^*||_p$ indicates larger perturbations can be allowed given a tolerance of loss increase $<c$, which is desired.

\subsection{Finding Perturbation Lower Bound}
\label{ssec:bound}

With a sufficiently small perturbation $\delta$, we can use Taylor expansion to well approximate the loss increase under weight perturbation with a quadratic function:
\begin{equation}
    L_{\sS}(\mW+\delta)-L_{\sS}(\mW) \approx \nabla_{\mW} L_{\sS}(\mW)^T \delta + \frac{1}{2} \delta^T \mH \delta,
\end{equation}
where $\nabla_{\mW} L_{\sS}(\mW)$ and $\mH$ denote the gradient and Hessian of the loss with respect to the weight $\mW$, respectively. For simplicity, in the rest of the section, we denote $g := \nabla_{\mW} L_{\sS}(\mW)$. We can thus rewrite the objective in Equation~(\ref{eq:perturb}) as: 
\begin{equation}
\label{eq:obj}
    \delta^* := \arg\min_{\delta} ||\delta||_p \ \ \  s.t.\ g^T \delta + \frac{1}{2} \delta^T \mH \delta \geq c,
\end{equation}

In the following discussion, we provide the lower bound on the minimal $||\delta^*||_2$ and $||\delta^*||_\infty$ needed to induce a loss increase of $c$ with respect to the properties of the loss function at weight $\mW$. The bounds on the magnitude of other $\ell_p$ norm bounded weight perturbations can be similarly derived from our result using the equivalence of norms in finite-dimensional spaces.

\begin{theorem}
\label{the:l2}
Assume that $v := \lambda_{max}(H) \geq 0$ as the largest eigenvalue of the Hessian, and $n := ||\mW||_0$ as the number of nonzero elements in $\mW$, we have
\begin{equation}
\label{eq:theL2}
    \frac{||g||_2}{v} \left( \sqrt{1+\frac{2vc}{||g||_2^2}} -1 \right) \leq ||\delta^*||_2,
\end{equation}
\begin{equation}
\label{eq:theLi}
    \frac{|g|}{nv} \left( \sqrt{1+\frac{2nvc}{|g|^2}} -1 \right) \leq ||\delta^*||_\infty.
\end{equation}
\end{theorem}

\begin{proof}
Here we follow the proof derived by~\citet{moosavi2019robustness} to their Theorem 1.
For any perturbation $\delta$ with $||\delta||_2 = r$ satisfying the condition in Equation~(\ref{eq:obj}) with $p=2$, we have
\begin{equation}
\label{eq:L2}
    -c + ||g||_2 r + \frac{v}{2} r^2 \geq -c + g^T \delta + \frac{1}{2} \delta^T \mH \delta \geq 0.
\end{equation}
This is derived from the Cauchy-Schwartz Inequality $||g||_2 ||\delta||_2 \geq g^T \delta$ and the Min-max Theorem $v ||\delta||_2^2 \geq \delta^T \mH \delta$.

Solving the second-order inequality in Equation~(\ref{eq:L2}), and considering the fact that $r \geq 0$, we have
\begin{equation}
    \frac{||g||_2}{v} \left( \sqrt{1+\frac{2vc}{||g||_2^2}} -1 \right) \leq r,
\end{equation}
which also hold when $r = ||\delta^*||_2$, leading to Equation~(\ref{eq:theL2})

Similarly, for any perturbation $\delta$ with $||\delta||_\infty = r$ satisfying the condition in Equation~(\ref{eq:obj}) with $p=\infty$, we have
\begin{equation}
\label{eq:Li}
    -c + |g| r + \frac{v}{2} r^2 n \geq -c + g^T \delta + \frac{1}{2} \delta^T \mH \delta \geq 0.
\end{equation}
This is derived from the Cauchy-Schwartz Inequality $|g|\ ||\delta||_\infty \geq g^T \delta$, the Min-max Theorem $v ||\delta||_2^2 \geq \delta^T \mH \delta$ and the the equivalence of norms in finite-dimensional spaces $\sqrt{n} ||\delta||_\infty \geq ||\delta||_2$.

Solving the second-order inequality in Equation~(\ref{eq:Li}), and considering the fact that $r \geq 0$, we have
\begin{equation}
    \frac{|g|}{nv} \left( \sqrt{1+\frac{2nvc}{|g|^2}} -1 \right) \leq r,
\end{equation}
which also hold when $r = ||\delta^*||_\infty$, leading to Equation~(\ref{eq:theLi}).

\end{proof}

Note that the lower bounds in Equations~(\ref{eq:theL2}) and~(\ref{eq:theLi}) both monotonically increase with the decrease of $v$, i.e., a smaller Hessian eigenvalue. This implies that under second-order approximation, having small Hessian eigenvalues is beneficial in limiting the loss increase under $\ell_p$ bounded weight perturbation, therefore inducing better generalization and quantization performance.

Interestingly, note that the bound in Equation~(\ref{eq:theLi}) is also monotonically increasing with decreasing $|g|$, showing the effectiveness of the previously proposed gradient $\ell_1$ regularization~\cite{alizadeh2020gradient}. Meanwhile, even if we consider the case where gradient $\ell_1$ is fully optimized, i.e., $|g|\rightarrow 0$, we have the lower bound
\begin{equation}
    \lim_{|g|\rightarrow 0} \left[ \frac{|g|}{nv} \left( \sqrt{1+\frac{2nvc}{|g|^2}} -1 \right) \right] = \sqrt{\frac{2c}{nv}},
\end{equation}
which may still be small if the Hessian eigenvalue $v$ is large. This analysis unveils that optimizing gradient $\ell_1$ is inadequate for the model's robustness against quantization, while further minimizing Hessian eigenvalues provides a stronger guarantee.

\section{Hessian-enhanced Training}

\subsection{Regularizing Hessian Eigenvalues}
\label{ssec:reg}

Following the conclusion of Theorem~\ref{the:l2}, here we aim to propose a regularization term that can minimize the squared sum of the Hessian matrix $\mH$'s eigenvalues $\lambda_i$ to encourage all eigenvalues to be small, thus minimizing the maximum eigenvalue $v$. This leads to our regularizer formulation:
\begin{equation}
\label{eq:reg_raw}
    L_r = \sum_i \lambda_i^2 = \E_{\vz} ||\mH \vz||^2, \ \vz \sim \mathcal{N}(0,\mI).
\end{equation}

With a finite difference approximation of the Hessian, we have $\mH \vz \approx \frac{\nabla L(\mW+hz)-\nabla L(\mW)}{h}$, where $h$ is a small positive number. Note that sampling multiple $z$ from the Gaussian distribution to compute the expectation may be costly; thus, we follow the observation made in CURE~\cite{moosavi2019robustness}, where the regularization loss can be estimated by only focusing on selected directions leading to high curvature, which often occurs along the gradient direction, i.e., $z = \nabla L(\mW)$~\cite{fawzi2018empirical,moosavi2019robustness}.
Thus we can convert the regularization term in Equation~(\ref{eq:reg_raw}) into 
\begin{equation}
    L_r(\mW) = ||\nabla L(\mW+h \vz)-\nabla L(\mW) ||^2,\ \vz = \nabla L(\mW)
\end{equation}
where $h>0$ is a small parameter determining the step size of the perturbation, and the $\frac{1}{h^2}$ term can be omitted by absorbing into the regularization strength parameter.

For a DNN model, $L_r$ needs to be computed on the weight tensors from all the layers, each having distinct dimensions and gradient value ranges. To accommodate the diversity among layers, we propose to compute $L_r$ in a layer-wise fashion, and scale the $\ell_2$ norm of the perturbation $\vz$ to match the weight value range in each layer. Specifically, for layer $i$ we have
\begin{equation}
\begin{split}
  \label{eq:reg}
    L_r^i(\mW^i) = ||\nabla L(\mW^i+h \vz^i)-\nabla L(\mW^i) ||^2 &, \\
     \vz^i = \frac{\mW^{i^2}}{||\mW^i||_2} \frac{\nabla L(\mW^i)}{||\nabla L(\mW^i)||_2} &.  
\end{split}
\end{equation}

The overall Hessian regularization is therefore computed as $L_r(\mW) = \sum_{i=1}^N L_r^i(\mW^i)$, summing over all the $N$ layers in the model.

\subsection{Hessian-enhanced Robust Optimization}
\label{ssec:opt}

In order to minimize $L_r(\mW)$ during DNN training, we provide an efficient and effective method to compute the gradient of $L_r^i(\mW^i)$ with respect to $\mW^i$. We start with defining $G(\mU) := ||\nabla L(\mU)-\nabla L(\mW^i) ||^2$, which allow us to convert $\nabla L_r^i(\mW^i)$ to
\begin{equation}
\begin{split}
  \label{eq:grad}
    \nabla L_r^i(\mW^i) & = \nabla_{(\mW^i+h \vz^i)} G(\mW^i+h \vz^i) \cdot \nabla_{\mW^i} (\mW^i+h \vz^i) \\& \approx \nabla_{(\mW^i+h \vz^i)} G(\mW^i+h \vz^i).  
\end{split}
\end{equation}

With this conversion, our regularization can be optimized with only one additional back propagation on the gradient difference $G$ with respect to the perturbed weight $\mW^i+h \vz^i$, which is well supported by common deep learning libraries such as TensorFlow and PyTorch. Note that we discard the second-order term $\nabla_{\mW^i} (\vz^i)$ in the final derivation step, which has been proven to be an effective approximation by~\cite{foret2020sharpness}.

In the meantime, note that regularizing the Hessian eigenvalue is necessary yet insufficient for the robustness against generalization and quantization. Since the Hessian regularization only regularizes the second-order derivative but not the first-order one, the final ``optimum'' may end up on a flat but steep slope in the loss surface. Adding a first-order regularization on the gradient norm is needed to mitigate the problem and complete the robust optimization. However, directly adding the $\ell_p$ norm of the gradient to the overall loss function requires additional computation and an additional regularization strength parameter. So instead, we take inspiration from the previous sharpness-aware minimization (SAM) method~\cite{foret2020sharpness}, which shows replacing the gradient of the original weight $\nabla_{\mW^i} L(\mW^i)$ with the gradient of the perturbed weight $\nabla_{(\mW^i+h \vz^i)} L(\mW^i+h \vz^i)$ in the SGD process effectively serves as a first-order regularization on the gradient norm and loss sharpness. This replacement can be made without additional cost as we already have $\nabla_{(\mW^i+h \vz^i)} L(\mW^i+h \vz^i)$ computed in the computation of $L_r^i(\mW^i)$.

With the approximation in Equation~(\ref{eq:grad}) and the addition of the first-order regularization in the SGD process, we can derive the gradient of our Hessian-enhanced robust optimization as
\begin{equation}
\label{eq:grad_final}
    \nabla_{\mW^i} = \nabla_{(\mW^i+h \vz^i)} L(\mW^i+h \vz^i) + \alpha \mW + \gamma \sum_{i=1}^N \nabla_{(\mW^i+h \vz^i)} G(\mW^i+h \vz^i),
\end{equation}
where $\alpha>0$ denotes the weight decay and $\gamma>0$ denotes the regularization strength of the Hessian regularization.
Performing SGD optimization with the derived gradient $\nabla_\mW$ in Equation~(\ref{eq:grad_final}) leads to the HERO algorithm, as illustrated in detail in Algorithm~\ref{alg}. 
\vspace{-21pt}
\begin{algorithm}
\small
	\caption{Hessian-Enhanced Robust Optimization (HERO)} 
	\label{alg}
	\begin{algorithmic}[1]
		\State Randomly initialize model weights $\mW^i_0$ for all layer $i$;
		\State Set total step $T$, perturbation strength $h$, learning rate $\eta$; 
		\State Set weight decay $\alpha$ and Hessian regularization strength $\gamma$;
		\For {$t=0,\ldots,T$}
			    \State Sample batch $\mathcal{B}$ from training set;
			    \State Compute batch loss's gradient $g^i = \nabla L_{\mathcal{B}}(\mW^i_t)$;
			    \State Compute weight perturbation $\vz^i$ with $g^i$ per Equation~(\ref{eq:reg});
			    \State Weight perturbation $\mW^{i*} = \mW^i+h \vz^i$;
			    \State Compute perturbed gradient $\nabla L_{\mathcal{B}}(\mW^{i*})$ 
			    \State Hessian regularization $G(\mW^{i*}) = ||\nabla L_{\mathcal{B}}(\mW^{i*})-g^i||^2$
			    \State Compute HERO gradient $\nabla_{\mW^i}$ per Equation~(\ref{eq:grad_final});
			    \State Weight update $\mW^i_{t+1} = \mW^i_t - \eta \nabla_{\mW^i}$
		\EndFor
		\Return $\mW^i_T$ for all layer $i$
	\end{algorithmic} 
\end{algorithm}
\vspace{-15pt}

\section{Evaluation}
\subsection{Experiment Setup}
We evaluate HERO with three representative DNNs: ResNet20~\cite{he2016deep}, MobileNetV2~\cite{sandler2018mobilenetv2}, and VGG19 with batch normalization~(VGG19BN) \cite{simonyan2014very} on the CIFAR-10 and CIFAR-100 datasets~\cite{krizhevsky2009learning}. The parameter numbers of these networks are $0.27M$~(ResNet20), $2.30M$~(MobileNetV2), and  $20.04M$~(VGG19BN). We further evaluate HERO with ResNet18~\cite{he2016deep} using the ImageNet dataset~\cite{deng2009imagenet} to validate the scalability of our method. The parameter number of ResNet18 is $11.17M$. We compare our approach with the stochastic gradient descent~(SGD) and Gradient $\ell 1$ (GRAD L1)~\cite{alizadeh2020gradient} training methods. We include GRAD L1 as a baseline because it is by far the state-of-the-art regularization method towards quantization robustness, yet only uses the first-order information of the quantization loss, in contrast to the second-order information used by HERO. 

All methods utilize a cosine learning rate scheduler with an initial learning rate~$\eta$ of $0.1$. We set the momentum as $0.9$ and the weight decay~$\alpha$ as $10^{-4}$. For the CIFAR-10 and CIFAR-100 experiments, we apply basic data augmentations, such as random crop, padding, and random horizontal flip on the training set, and train the model for $200$ epochs with batch size $128$. For the ImageNet experiments, random resized crop and normalization are applied to the training set. We train the model for $100$ epochs with batch size $256$. Note that we train the model from scratch in all the experiments. All experiments are conducted using NVIDIA TITAN RTX GPUs.

For HERO, to select the Hessian regularization strength~$\gamma$, we conduct a grid search over $\{0.01, 0.05, 0.1, 0.5, 1.0, 5.0\}$.
For the weight perturbation step size~$h$, we follow the previous experiment settings in~\cite{foret2020sharpness} to utilize $0.5$ for CIFAR-10 experiments and $1.0$ for other experiments.
For the GRAD L1 regularization strength, we follow the steps in~\cite{alizadeh2020gradient} to run a grid search to find the best hyperparameter with the minimal sacrifice of the test accuracy.

\subsection{Improving Model Generalization}
\label{ssec:generalize}

As discussed in Theorem~\ref{the:gen} and Equation~(\ref{eq:theL2}), HERO is beneficial on limiting the loss increase under $\ell_2$ bounded weight perturbation, thus realizing better generalization performance. In this subsection, we showcase HERO's effectiveness in improving model generalizability with experiments on the test accuracy comparison and the noisy-label training performance.

\begin{table}[t]
\caption{Test accuracy on various models and datasets.}
\vspace{-9pt}
\label{table:generalization}
\begin{tabular}{c c c c c}
\toprule
Dataset                   & Model       & HERO    & GRAD L1   & SGD     \\ \midrule
\multirow{3}{*}{CIFAR-10}  & ResNet20    & \textbf{93.44\%} & 92.82\% & 92.82\% \\ 
                          & MobileNetV2 & \textbf{95.03\%} & 92.52\% & 92.45\% \\
                          
                          & VGG19BN & \textbf{94.79\%} & 93.41\% & 93.89\% \\
                          \hline
\multirow{3}{*}{CIFAR-100} & ResNet20    & \textbf{70.72\%} & 69.30\% & 69.52\% \\ 
                          & MobileNetV2 & \textbf{76.90\%} & 74.13\% & 73.12\% \\
                          
                          & VGG19BN & \textbf{76.09\%} & 74.05\% & 74.61\% \\
            \hline
ImageNet & ResNet18 & \textbf{71.05\%} & 70.82\% & 70.74\% \\ \bottomrule
\end{tabular}
\vspace{-15pt}
\end{table}
\noindent\textbf{Test Accuracy.} 
We evaluate the test accuracy of HERO and baseline methods in Table~\ref{table:generalization}. 
For ResNet20, MobileNetV2 and VGG19BN models, HERO achieves $0.62\%$, $2.58\%$ and $0.90\%$ accuracy gain compared with SGD on CIFAR-10 dataset respectively. 
For experiments on the CIFAR-100 dataset, HERO can reach the accuracy of $70.72\%$, $76.90\%$, and $76.09\%$ in ResNet20, MobileNetV2, and VGG19BN network, which increases the performance by $1.20\%$, $3.78\%$, and $1.48\%$ with respect to SGD.
One thing worth noting is that HERO enables a better test accuracy on compact models without enlarging the network size. For instance, on the CIFAR-10 and CIFAR-100 dataset, the MobileNetV2 test accuracy achieved by HERO can outperform the VGG19 test accuracy achieved by SGD, with $\sim 8.7\times$ fewer parameters. This further benefits the deployment of efficient models in the real world.

On the contrary, we find GRAD\ L1 method, which is designed against $\ell_\infty$ bounded weight perturbation, doesn't guarantee a consistent improvement of the test accuracy against SGD. This implies that generalizing the robustness against $\ell_\infty$ bounded to $\ell_2$ bounded weight perturbation isn't trivial. On the other hand, HERO provides both consistent promising generalization performance and robustness against quantization, as further discussed in Section~\ref{ssec:quant}.

To further validate the scalability of HERO, we test with the ResNet18 model on ImageNet. The result confirms that HERO can improve the generalization compared to GRAD L1 and SGD.
\begin{table}[t]
\caption{Test accuracy under noisy-label training.}
\vspace{-12pt}
\label{table:noisy-label}
\centering
(a) ResNet20\\
\begin{tabular}{c c c c c}
\toprule
Noise ratio & 20\%  & 40\%  & 60\%  & 80\%  \\ \midrule
HERO                              & \textbf{90.63\%} & \textbf{88.71\%} & \textbf{84.61\%} & \textbf{72.11\%} \\ 
GRAD L1                             & 85.91\% & 78.66\% & 65.86\% & 48.28\% \\ 
SGD                               & 85.64\% & 78.73\% & 66.42\% & 42.17\% \\ \bottomrule
\end{tabular}\\
\vspace{4pt}
\centering
(b) MobileNetV2\\
\begin{tabular}{c c c c c}
\toprule
Noise ratio & 20\%    & 40\%    & 60\%  & 80\%  \\ \midrule
HERO                              & \textbf{91.70\%} & \textbf{88.57\%} & \textbf{81.73\%} & \textbf{72.03\%} \\ 
GRAD L1                             & 89.00\%   & 85.56\%   & 79.73\% & 30.34\% \\ 
SGD                               & 89.28\%   & 85.84\%   & 80.49\% & 62.91\% \\ \bottomrule
\end{tabular}
\vspace{-12pt}
\end{table}
\noindent\textbf{Noisy-Label Training.} For models trained on real-world data, inevitable label noise will exist in the training dataset. Robustness against noisy labels in the training process is essential for the model's generalizability to the test data. Here we show that HERO is still robust under the presence of noisy labels.

We utilize ResNet20 and MobileNetV2 networks on the CIFAR-10 dataset for this experiment. First, we follow the symmetric noisy label generation in~\cite{li2020dividemix}, where we uniformly sample a certain proportion~(from $20\%$ to $80\%$, namely \textit{noise ratio}) of the training data and replace their labels with a uniform random sample from all the possible classes. We then train the model with the same training procedure on the perturbed training set, and evaluate the accuracy on the original clean test set. 
As shown in Table~\ref{table:noisy-label}, HERO has the best test accuracy across all noise ratios among all three methods. Besides, the test accuracy of GRAD L1 and SGD drops dramatically at the high noise ratio of $80\%$; while the HERO approach still provides acceptable results. Therefore, HERO shows its robustness against the training label perturbation and achieve the best performance under noisy training label among all methods.
\subsection{Improving Quantization Robustness}
\label{ssec:quant}
In Theorem~\ref{the:quant} and Equation~(\ref{the:quant}), we show that the loss change of uniform weight quantization is bounded by the model performance under a weight perturbation $\delta$ bounded by its $\ell_{\infty}$ norm, where lower quantization precision indicates a higher weight perturbation. 
Here we demonstrate the quantization robustness achieved by HERO with the post-training quantization to various precision. No quantization-aware finetuning is performed in these experiments.
\begin{figure}[t]
\centering
\includegraphics[width=\linewidth]{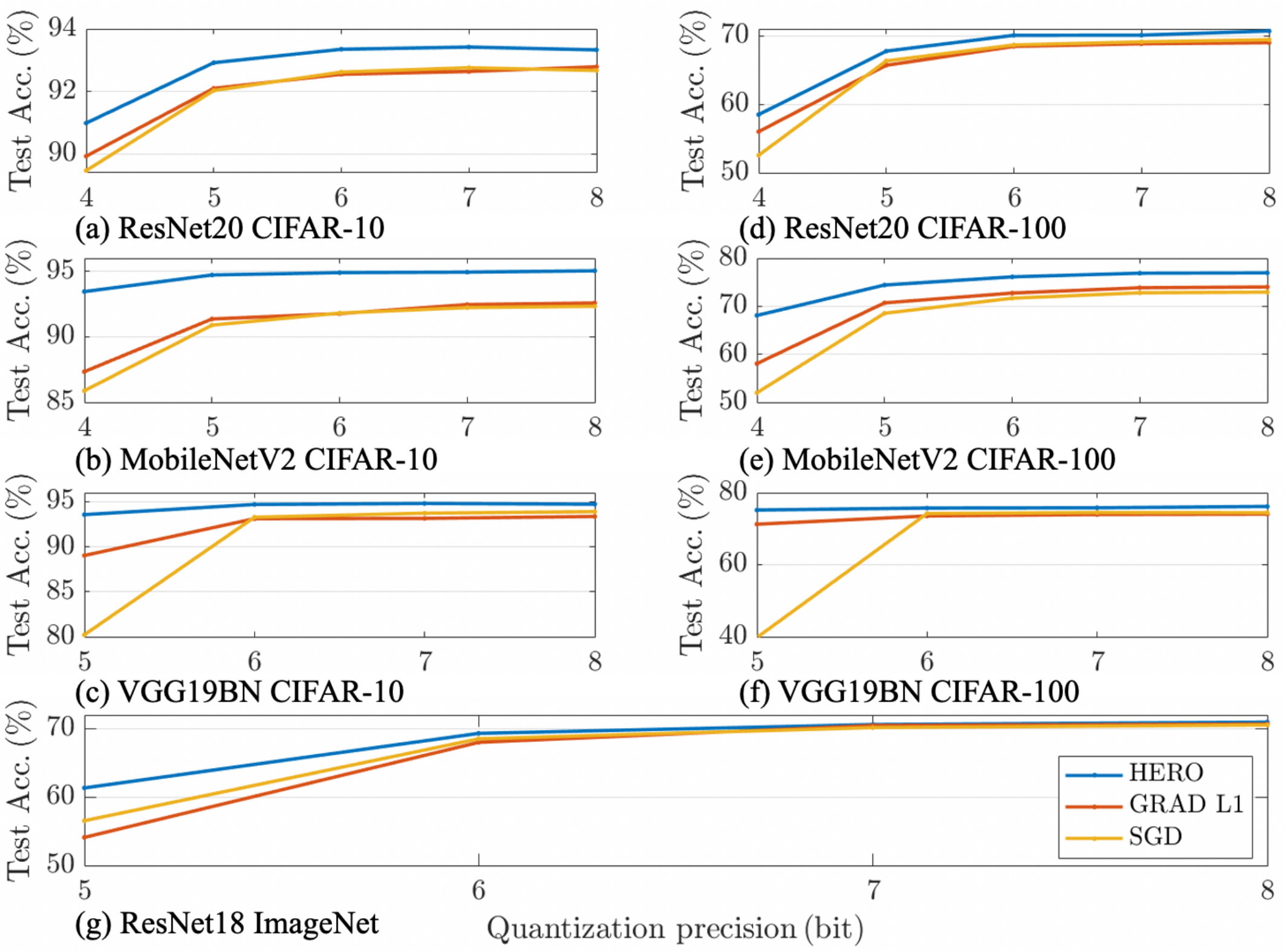}
\vspace{-18pt}
\caption{Post-training quantization accuracies with HERO, GRAD L1, and SGD: (a)-(c) ResNet20, MobileNetV2, VGG19BN experiments on CIFAR-10 dataset; (d)-(f) ResNet20, MobileNetV2, VGG19BN experiments on CIFAR-100 dataset; (g) ResNet18 experiments on ImageNet dataset.}
\label{fig:quantization}
\vspace{-12pt}
\end{figure}

The experiments on the CIFAR-10 dataset are shown in Figure~\ref{fig:quantization}~(a)-(c). The test accuracy for HERO across different quantization precision is consistently higher than that of GRAD L1 and SGD. Our observation matches with~\cite{alizadeh2020gradient} that GRAD\ L1 can achieve better test accuracy to some extent under low weight precision compared with SGD. 
Yet, the second-order regularization introduced by HERO provides a better guarantee of quantization robustness. 

More significantly, the HERO performance under low quantization precision shows a large improvement compared with baselines across all the precision. For instance, for the MobilenetV2 network, test accuracy for HERO under 4-bit weight is $93.45\%$, significantly higher than the $87.34\%$ and $85.88\%$ achieved by GRAD L1 and SGD, respectively. DNN quantization with ultra-low precision is a challenging problem due to large perturbations on the weights, while HERO effectively provides robustness against such perturbation.

We also notice that a model with more parameters is more sensitive to quantization perturbation. In our case of the VGG19BN network, SGD with 5-bit quantization already leads to noticeable accuracy degradation compared to full precision results. 
In the meantime, HERO still retains a $93.57\%$ test accuracy compared to the $89.03\%$ and $80.22\%$ accuracy of GRAD L1 and SGD, showing its effectiveness on larger models. 

A similar trend can also be observed on other datasets.
On the CIFAR-100 dataset, as shown in Figure~\ref{fig:quantization}~(d)-(f), the consistent trend that HERO outperforms GRAD\ L1 and SGD still holds across different quantization precision. Besides, in the low precision setting, HERO has an outstanding performance gain compared with baseline methods.
For instance, on the MobileNetV2 network, HERO improves the test accuracy under 4-bit quantization by $10.05\%$ and $16.10\%$ compared with GRAD\ L1 and SGD, respectively. Our quantization result with ResNet18 on ImageNet dataset also shows that HERO can provide better quantization robustness across different quantization precision, as shown in Figure~\ref{fig:quantization}~(g).
\subsection{Theoretical Insight Verification}
\begin{figure}[t]
\centering
\includegraphics[width=0.98\linewidth]{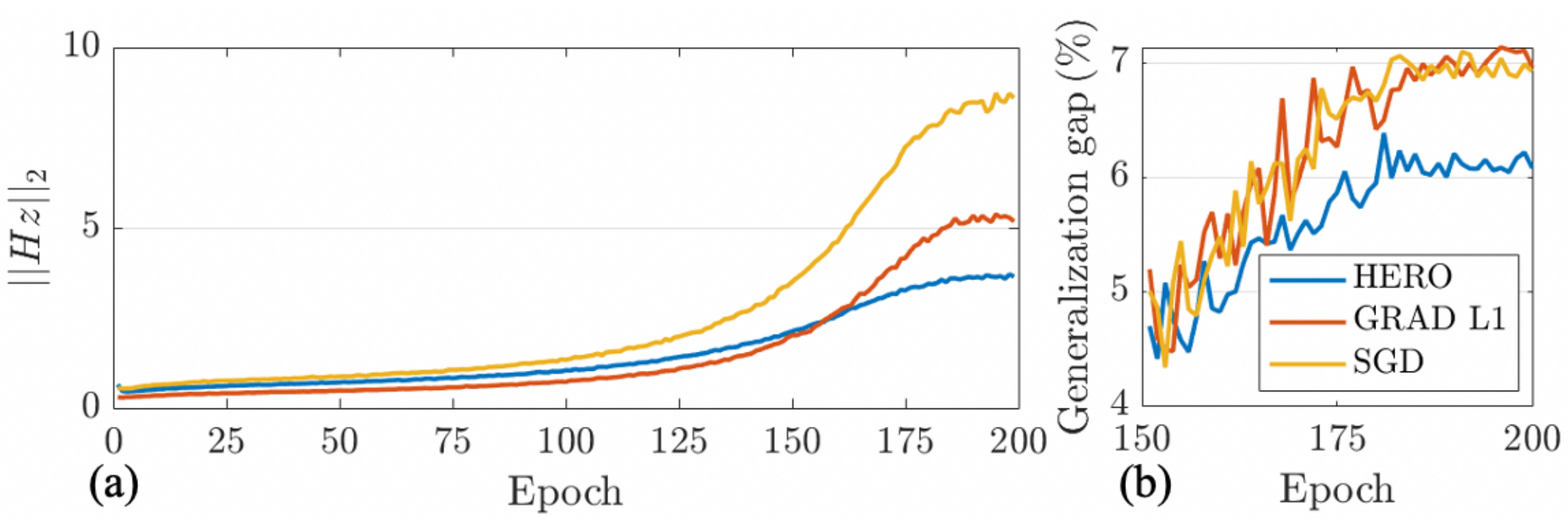}
\vspace{-10pt}
\caption{Hessian norm and generalization gap evolution through the training with HERO, GRAD\ L1, and SGD.}
\vspace{-10pt}
\label{fig:hessian}
\end{figure}
\noindent\textbf{Hessian Norm across Training Process.}
To show the effectiveness of HERO on regularizing Hessian eigenvalues, we visualize the evolution of the Hessian norm $||\mH \vz||_2$ throughout the training process in Figure~\ref{fig:hessian}~(a) following the setting in~\cite{moosavi2019robustness}, with $\vz$ being the perturbation defined in Equation~(\ref{eq:reg}). All curves are averaged over the entire CIFAR-10 training set. 
The generalization gap between training and test accuracy in the final 50 training epochs is shown in Figure~\ref{fig:hessian}~(b).
Throughout the training process, the Hessian norm gets larger as the model overfits to the training set. Meanwhile, HERO helps keep the Hessian norm values at the lowest level towards the end of the training process, and thus leads to the lowest generalization gap as expected.

\noindent\textbf{Loss Contour Visualization.}
\begin{figure}[t]
\centering
\includegraphics[width=0.99\linewidth]{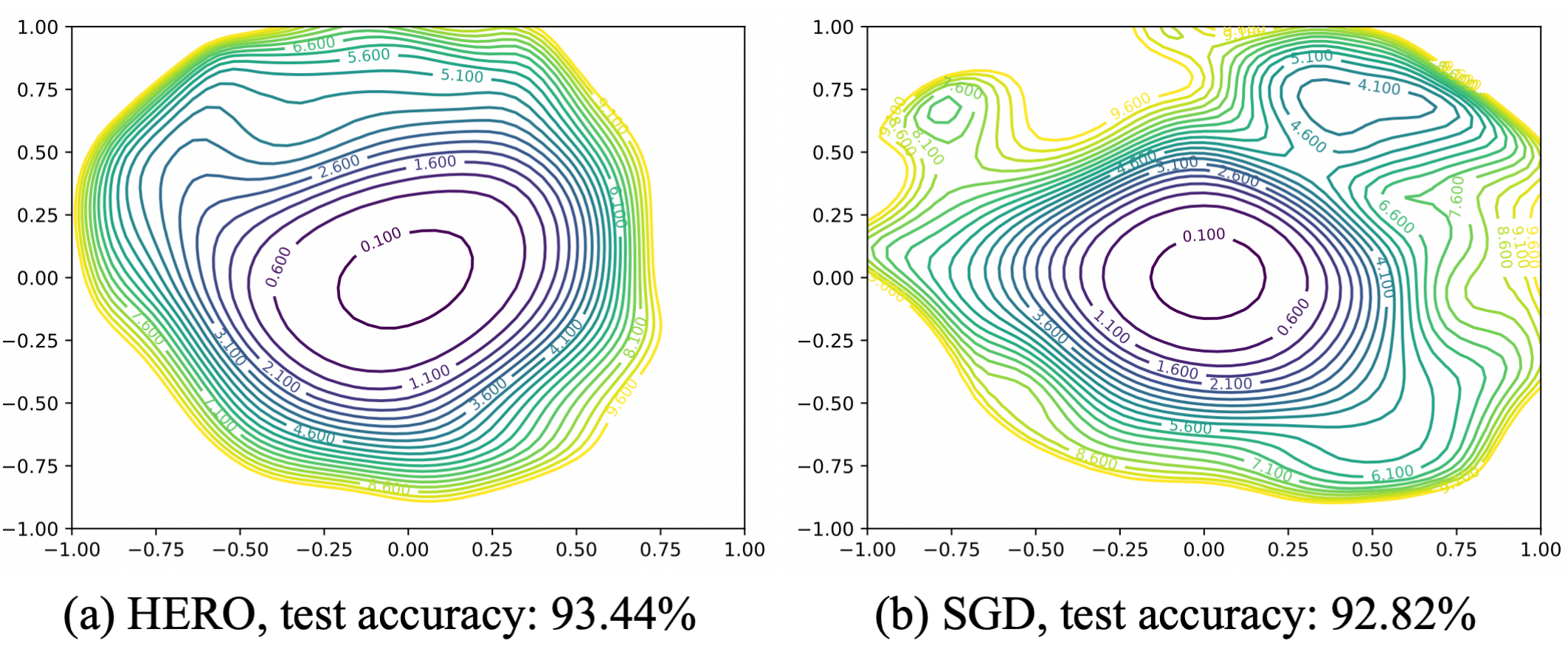}
\vspace{-12pt}
\caption{Loss surface contour along 2 random directions around converged weights. Estimated on ResNet20 model on CIFAR-10 dataset trained with HERO and SGD.}
\label{fig:2d}
\end{figure}
We further showcase the weight perturbation robustness achieved by HERO with the loss contour in the neighborhood of the converged weights, with HERO in Figure~\ref{fig:2d}~(a) and SGD in Figure~\ref{fig:2d}~(b) plotted under the same scale. The contour is generated with the visualization tool provided by~\cite{li2017visualizing}, which monitors the loss change while applying normalized adjustments to the weight along two random directions. Compared to that of SGD, the loss surface of HERO appears to be smoother, with a larger region within the inner contour circle indicating a 0.1 loss increase. This shows HERO is robust to larger perturbation within a tolerance of loss increase, which is well in line with Theorem~\ref{the:l2}.  

\begin{table}[t]
\caption{Ablation study on HERO, first-order only, and SGD gradient update rule. Results reported with MobileNetV2 network on CIFAR-10 dataset.}
\vspace{-9pt}
\label{table:ablation}
\centering
\begin{tabular}{c c c c c}
\toprule
Quantization (bit)  & 4       & 6       & 8       & Full      \\ \midrule
HERO                & \textbf{93.45\%} & \textbf{94.90\%} & \textbf{95.03\%} & \textbf{95.03\%} \\ 
First-order only    & 91.61\% & 93.92\% & 94.00\% & 94.06\% \\ 
SGD                 & 85.88\% & 91.81\% & 92.33\% & 92.45\% \\ \bottomrule
\end{tabular}
\vspace{-10pt}
\end{table}
\noindent\textbf{The Necessity of Hessian-enhanced Method.}
In the derivation of HERO's gradient in Equation~(\ref{eq:grad_final}), we borrow the first-order regularization~($ \nabla_\mW = \nabla_{(\mW^i+h \vz^i)} L(\mW^i+h \vz^i)$) from SAM~\cite{foret2020sharpness} alongside our Hessian regularization. SAM itself already leads to the state-of-the-art result on generalization performance~\cite{foret2020sharpness}, beating dropout~\cite{srivastava2014dropout} and Mixup~\cite{zhang2017mixup}. Here we show the Hessian regularization introduced by HERO is still necessary, as it further increases the generalization and quantization performance over SAM.
We compare HERO with the first-order only method~(i.e., $ \nabla_{\mW^i} = \nabla_{(\mW^i+h \vz^i)} L(\mW^i+h \vz^i)+\alpha \mW$) and SGD~(i.e., $ \nabla_{\mW^i} = \nabla_{\mW^i} L(\mW^i)+\alpha \mW$) in Table~\ref{table:ablation}. For test accuracy on the full precision model, HERO provides an additional $1\%$ gain over the first-order only method. 
Furthermore, HERO provides better robustness against quantization.
For example, 4-bit weight quantization with the HERO model leads to a $1.6\%$ accuracy drop, much smaller than the $2.5\%$ drop achieved with the first-order regularization. The result confirms the necessity of including the Hessian regularization in the pursuit of both generalization and quantization performance.

\section{Conclusion}

This work proposes HERO, a Hessian-enhanced robust optimization method to improve the generalization and quantization performance of DNN models simultaneously. We provide novel insights on unifying generalization and quantization under improving weight perturbation robustness, theoretical analysis on enhancing the robustness with Hessian regularization, and empirical results showing the effectiveness of HERO. We hope this work helps on deploying DNN models onto real-world mobile and edge devices, and inspires further attention to the robustness against weight perturbation.

\bibliographystyle{ACM-Reference-Format}
\bibliography{main}
\end{document}